\documentclass[11pt,letterpaper]{article}
\usepackage[top=1in, bottom=1in, left=1in, right=1in]{geometry}

\usepackage[utf8]{inputenc} % allow utf-8 input
\usepackage[T1]{fontenc}    % use 8-bit T1 fonts
\usepackage{lmodern}
\usepackage{hyperref}       % hyperlinks
\usepackage{url}            % simple URL typesetting
\usepackage{booktabs}       % professional-quality tables
\usepackage{amsfonts}       % blackboard math symbols
\usepackage{nicefrac}       % compact symbols for 1/2, etc.
\usepackage{microtype}      % microtypography
\usepackage{xcolor}         % colors
\usepackage{natbib}
 
\usepackage{amsmath, amsthm, amssymb}
\usepackage{bm}
\usepackage{color}
\usepackage{ulem}

\let\emph\textit

\def\R{{\mathbb{R}}}

\def\1{\textbf{1}}

\newcommand{\onevct}{\bm{1}}

\newcommand{\onemtx}{\onevct\onevct^\top}

\newtheorem{theorem}{Theorem}

\newtheorem{lemma}{Lemma}
\newtheorem{definition}{Definition}
\newtheorem{claim}{Claim}

\newcommand{\Prob}[2][]{\mathbb{P}_{#1}\left\{ {#2} \right\}}
\newcommand{\Expect}[2][]{\mathbb{E}_{#1}\left[ #2 \right]}

\newcommand{\abs}[1]{\left\vert {#1} \right\vert}

\newcommand{\norm}[1]{\left\Vert {#1} \right\Vert}

\newcommand{\sgn}[1]{\operatorname*{sgn}\left({#1}\right)}

\newcommand{\diag}[1]{\operatorname*{diag}\left({#1}\right)}
\newcommand{\tr}[1]{\operatorname*{tr}\left({#1}\right)}

\newcommand{\argmax}{\operatorname*{arg\; max}}

\newcommand{\st}{\operatorname*{subject\; to}}
\newcommand{\maximize}{\operatorname*{maximize}}
\newcommand{\minimize}{\operatorname*{minimize}}

\newcommand{\lsec}{\lambda_{2}}
\newcommand{\lperp}{\lambda_{\onevct^\perp}}

\newcommand{\inprod}[2][]{\left\langle {#1},{#2} \right\rangle}

\newcommand{\kth}{{(k)}}

\newcommand{\yyast}{y^\ast y^{\ast\top}}
\usepackage{silence}
\usepackage{enumitem}
\usepackage{multibib} 

\usepackage{algorithm}
\usepackage{algorithmic}
\usepackage{graphicx}
\usepackage{subcaption}
\usepackage{multicol}
\usepackage[font=small,labelfont=bf]{caption}

\makeatletter
\def\algbackskip{\hskip-\ALG@thistlm}
\makeatother

\title{Federated Myopic Community Detection with One-shot Communication}

\author{
  \textbf{Chuyang Ke}\\Department of Computer Science\\Purdue University\\\texttt{cke@purdue.edu}
  \and 
  \textbf{Jean Honorio}\\Department of Computer Science\\Purdue University\\\texttt{jhonorio@purdue.edu}
}

\date{}
\begin{document}
\maketitle

\begin{abstract}
  In this paper, we study the problem of recovering the community structure of a network under federated myopic learning. 
  Under this paradigm, we have several clients, each of them having a myopic view, i.e., observing a small subgraph of the network. 
  Each client sends a censored evidence graph to a central server. 
  We provide an efficient algorithm, which computes a consensus signed weighted graph from clients evidence, and recovers the underlying network structure in the central server. 
  We analyze the topological structure conditions of the network, as well as the signal and noise levels of the clients that allow for recovery of the network structure. 
  Our analysis shows that exact recovery is possible and can be achieved in polynomial time. 
  We also provide information-theoretic limits for the central server to recover the network structure from any single client evidence.
  Finally, as a byproduct of our analysis, we provide a novel Cheeger-type inequality for general signed weighted graphs.
\end{abstract}

\allowdisplaybreaks

\section{Introduction}
Modern social networks have underlying community structures \citep{javed2018community}. Take Twitter as an example: one may assume that Twitter users with similar political views are more likely to interact with each other.
\emph{Community detection} is the task of identifying underlying group structures in a network given observation of node interactions \citep{kelley2012defining}.
At the same time, modern social networks exhibit two opposite defining properties: big volume, small neighborhoods. In the Twitter example, more than $1.3$ billion Twitter accounts have been created as of the year 2019, yet the average number of followers is $707$ \citep{smith_2020}. Thus when studying modern social networks, it is impractical and inefficient to collect all data. 
Furthermore, due to growing privacy concerns and regulations, it is unfavorable for a central server to collect all users' information.

As an emerging technique in the machine learning community, \emph{federated learning} tries to address the issues discussed above. The idea of federated learning is not limited to any particular learning algorithm; rather, it is a learning paradigm, under which a central server trains a high-quality learning model with the coordination of a federation of participating clients \citep{konevcny2016federated}.  
In a typical federated learning setting, it is often assumed that the participating clients are large in number. At the same time, each client has a very small and non-i.i.d. (independent and identically distributed) dataset, very limited computational power, and very restricted communication capabilities. An example is mobile phones. Due to security reasons and storage limits, each mobile phone can only access data that are related to the device. The computational capability of each mobile phone is limited by its processor and battery power, thus it is not practical to run large-scale algorithms on each device. Furthermore, the communication between mobile phones and a central server may be restricted due to connection quality, privacy concerns, and government censorship.

In this paper we study the problem of community detection under a federated learning paradigm. We focus on a myopic setting, in which every client has a limited and non-identical access to the network. Each client observes a small subgraph of the whole network, and sends a censored evidence graph to a central server. The server then computes a consensus graph from the evidence sent by the clients, and recovers the underlying community structure of the full network. To model community interaction we adopt a generative approach similar to the one in the stochastic block model (SBM). In other words, we assume nodes from the same group are more likely to be connected than those from different groups. We try to answer the following questions:
\begin{itemize}
    \item Does there exist an efficient central server algorithm that takes the censored clients' local observation as the input, and recovers the underlying community structure of the \emph{full} network?
    \item Under what topological and statistical conditions, the central algorithm will work correctly?
\end{itemize}

The feasibility of efficient and correct community detection depends on the topological structure, as well as the signal and noise level of the network. The latter is intuitive: as in the case of SBM, if in-group interaction is dense and across-group interaction is sparse, the community structure is clear and thus recovery of the group membership can be achieved. 
On the other hand, the topological structure of the network is a unique challenge under the federated learning paradigm. By ``network topology'', we refer to the union of the subgraphs observed by all clients and the related properties. 

In community detection tasks, one of the most important properties is the edge expansion. In general, for undirected unweighted graphs, the edge expansion measures how connected every component of the network is. It is well-known that the Cheeger constant can be used to measure the edge expansion property of a graph.
The edge expansion is important for successful community detection. For example, if the subgraphs observed by all clients do not have any intersection at all, arguably it is not possible to correctly recover the full community structure, as there is no observed community interaction between nodes in different subgraphs (see Figure \ref{fig:link} for illustration). 
Similarly, if the union of the subgraphs is a chain graph or a star graph, it has a ``bottleneck''. In this case, removing very few edges will disconnect the graph (see Figure \ref{fig:star} and \ref{fig:star2} for illustration).
We find that the correctness of community detection heavily depends on the bottleneck edges. 
In our analysis we propose a novel Cheeger-type constant, which characterizes the edge expansion property of general signed weighted graphs (that is, weighted graphs with potentially negative weights). We show that the Cheeger-type constant of the server consensus graph is critical for correct recovery of the underlying community structure, along with the signal and noise level parameters.

%%%%%%%%% FIGURE BEGINS %%%%%%%%% 

\begin{figure*}[ht!]
    \centering
    
    \begin{subfigure}[b]{.25\linewidth}
    \centering
    \includegraphics[width=0.65\linewidth]{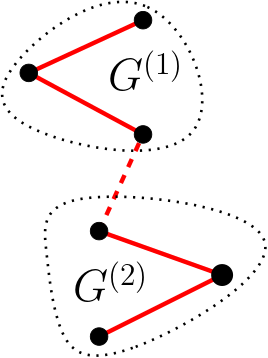}
    \caption{Chain graph (smallest edge expansion) }
    \label{fig:link}
    \end{subfigure}
    \hspace{10mm}
    \begin{subfigure}[b]{.25\linewidth}
    \includegraphics[width=\linewidth]{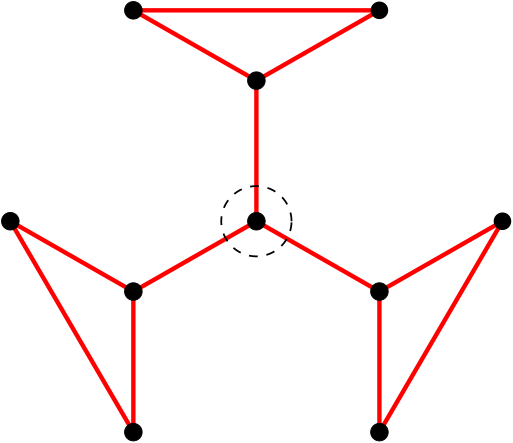}
    \caption{Star graph (median edge expansion)}
    \label{fig:star}
    \end{subfigure}
    \hspace{10mm}
    \begin{subfigure}[b]{.25\linewidth}
    \includegraphics[width=\linewidth]{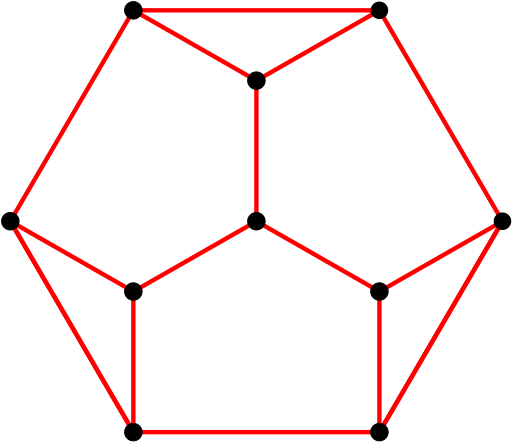}
    \caption{Regular graph (largest edge expansion)}
    \label{fig:star2}
    \end{subfigure}
    
    % \vspace{-.5mm}
    \caption{Graph topology examples. Figure \ref{fig:link} shows a chain graph, where two clients observe subgraph $G^{(1)}$ and $G^{(2)}$, respectively. In this case it is impossible for any algorithm to recover the full network structure, since there is no observed interaction between the two subgraphs.  
    Figure \ref{fig:star} shows a star graph, in which every edge connecting the highlighted node in the center is a ``bottleneck,'' i.e., removing any of the three edges disconnects the graph.
    Figure \ref{fig:star2} shows a $3$-regular graph, in which every component is connected with no ``bottleneck.''}
    \label{figs:edge_expansion}
    \end{figure*}
    
    %%%%%%%%% FIGURE ENDS %%%%%%%%% 

\textbf{Related Works.}
As a rapidly growing research area, federated learning was formally defined in \citet{konevcny2015federated, konevcny2016federated} as a general machine learning setting. 
Recently, both \citet{li2020federated} and \citet{yang2019federated} provided an overview of federated learning, and interestingly, both overviews highlighted that one main challenge in federated learning is the statistical heterogeneity of clients' local data. 
The framework of federated learning has been applied to various problems such as deep networks \citep{mcmahan2017communication}, principal component analysis \citep{Grammenos20} and sparse linear regression \citep{barik2020exact}, to mention a few.
On the other hand, federated learning for community detection tasks has not been studied yet. 
\citet{mercado2018power} studied a tangentially related problem called multilayer graph clustering, however the topological structure in that case is much simpler, and the weights are assumed to be nonnegative.
To the best of our knowledge, we are providing the first community detection model under a federated myopic learning setting, where the clients observe a non-identical small subgraph of the full network.

There are several different generalizations of signed graphs and the Laplacian \citep{bronski2014spectral}. One variation studied in \citet{kunegis2010spectral,cucuringu2019sponge, knyazev2017signed} is the so-called signed degree matrix. The idea is that instead of summing all edge weights directly, they sum over the absolute value of the edge weights, and their signed Laplacian matrix can be defined in a similar way. Readers should be aware that this is different from our approach, and more importantly, only our approach fulfills the Karush–Kuhn–Tucker (KKT) conditions required by our algorithm.   
There are also some works on the application of signed networks, for example, \citet{doreian2009partitioning} proposed a heuristic algorithm to partition signed social networks, and \citet{giotis2006correlation} proposed an approximation algorithm to maximize agreements in a graph. Our theorems focus on the spectral analysis of the signed weighted graph, and provide provable theoretical guarantees for efficient exact recovery. 

\textbf{Summary of Our Contribution.} Our work is mostly theoretical. We provide a series of novel results in this paper:
\begin{itemize}
    \item We introduce a highly-general federated myopic learning framework with one-shot communication for community detection tasks. Under this paradigm, every client observes a small subgraph of the full network and sends a censored evidence graph to a central server. The central server computes a consensus signed weighted graph, and recovers the underlying network structure.
    \item We provide provable theoretical guarantees for efficient exact recovery of the community structure. We analyze the topological structure conditions of the consensus graph, as well as the signal and noise levels of the clients that allow for recovery of the community structure. We establish the regime in which exact recovery is possible and can be achieved in polynomial time. We also provide information-theoretic limits for any algorithm to recover the local community structure from any single client evidence. 
    \item We propose a novel Cheeger-type inequality for general signed weighted graphs with potentially negative weights, from a graph theoretical point of view. The inequality relates the eigenvalue gap of a signed weighted graph to a Cheeger-type constant.
\end{itemize}
\section{Our Novel Problem Formulation}

%%%%%%%%% FIGURE BEGINS %%%%%%%%% 

\begin{figure*}[ht!]
\centering
\captionsetup[subfigure]{labelformat=empty}

\begin{subfigure}[b]{.25\linewidth}
\includegraphics[width=\linewidth]{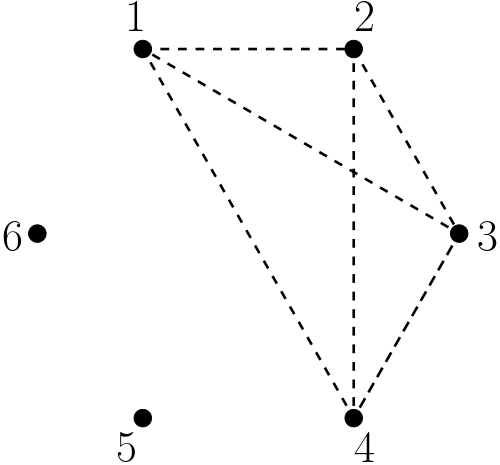}
\caption{$\Omega^{(1)}$: Field of view of client $1$}
\label{fig:e1}
\end{subfigure}
\hspace{10mm}
\begin{subfigure}[b]{.25\linewidth}
\includegraphics[width=\linewidth]{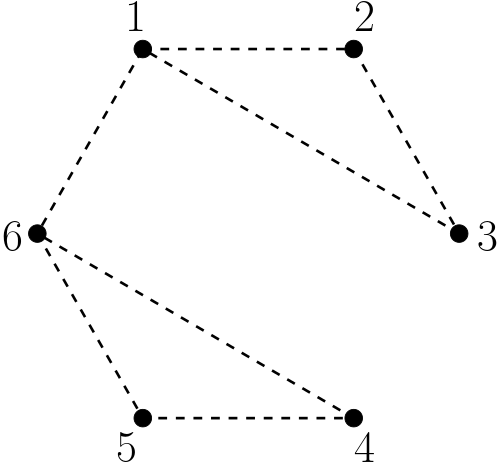}
\caption{$\Omega^{(2)}$: Field of view of client $2$}
\label{fig:e2}
\end{subfigure}
\hspace{10mm}
\begin{subfigure}[b]{.25\linewidth}
\includegraphics[width=\linewidth]{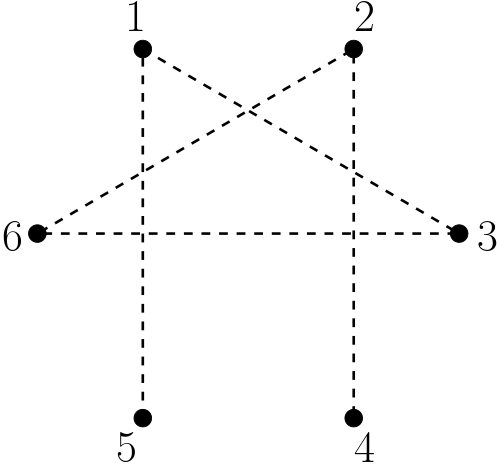}
\caption{$\Omega^{(3)}$: Field of view of client $3$}
\label{fig:e3}
\end{subfigure}

\par\bigskip
% \vspace{-2.5mm}

\begin{subfigure}[b]{.25\linewidth}
\includegraphics[width=\linewidth]{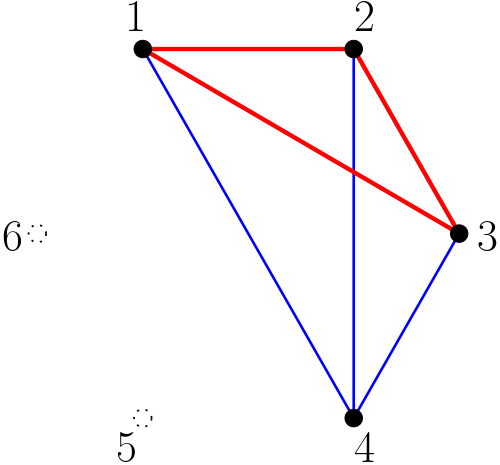}
\caption{$G^{(1)}$: Subgraph of client $1$}
\label{fig:a1}
\end{subfigure}
\hspace{7mm}
\begin{subfigure}[b]{.25\linewidth}
\includegraphics[width=\linewidth]{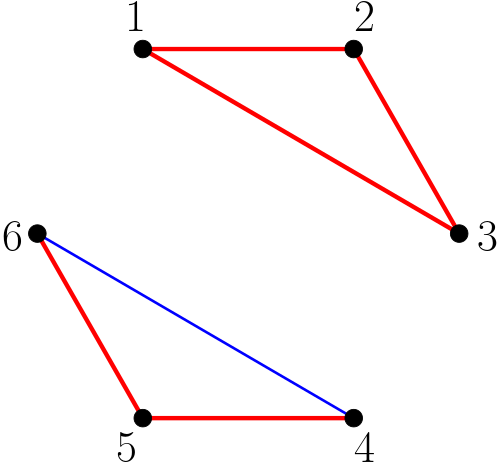}
\caption{$G^{(2)}$: Subgraph of client $2$}
\label{fig:a2}
\end{subfigure}
\hspace{7mm}
\begin{subfigure}[b]{.25\linewidth}
\includegraphics[width=\linewidth]{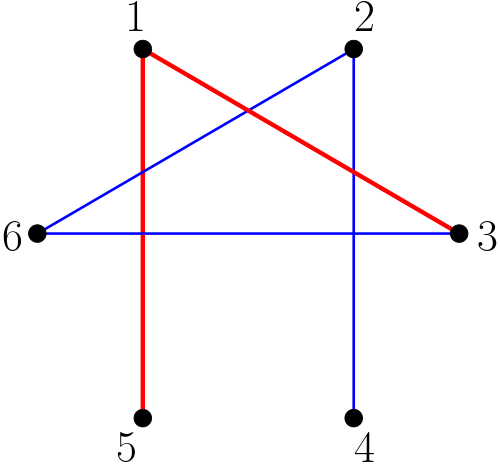}
\caption{$G^{(3)}$: Subgraph of client $3$}
\label{fig:a3}
\end{subfigure}

\par\bigskip
% \vspace{-2.5mm}

\begin{subfigure}[b]{.25\linewidth}
\includegraphics[width=\linewidth]{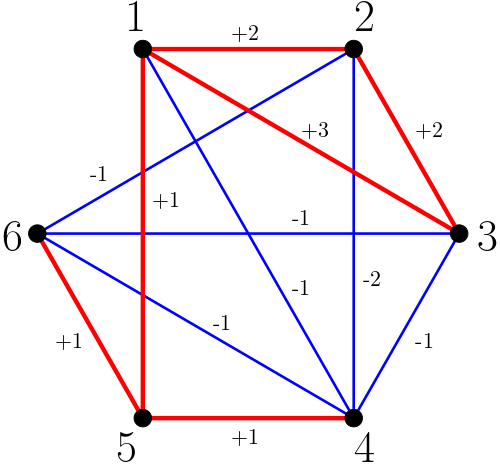}
\caption{$G$: Consensus signed weighted graph}
\label{fig:w}
\end{subfigure}

% \vspace{-.5mm}
\caption{Examples of client fields of view ($\Omega^{(1)}, \Omega^{(2)}, \Omega^{(3)}$), client-observed subgraphs ($G^{(1)}, G^{(2)}, G^{(3)}$), and server consensus signed weighted graph.
In FOVs, each viewable edge is represented by a dashed line. 
In subgraphs, each edge is colored red, and each non-edge is colored blue.
The consensus graph is the weighted summation of all subgraphs, such that each edge is counted as $+1$, and each non-edge is counted as $-1$.
In the particular examples above, it is assumed that clients do not censor evidence graphs, i.e., $\tilde{G}^\kth = G^\kth$, for simplicity of visualization. In the rest of the paper, we assume that clients censor evidence graphs.} 
\label{figs:client_server}
% \vspace{-3mm}
\end{figure*}

%%%%%%%%% FIGURE ENDS %%%%%%%%% 

In this section, we present an overview of federated myopic learning. 
We formally define the community detection task under this paradigm, and present an central server algorithm that solves the problem by using the censored subgraph information reported by the clients. 
We also introduce the notations that will be used later in the paper.

Without specification we use lowercase letters (e.g., $a,b,u,v$) for scalars and vectors, and uppercase letters (e.g., $A, B, C$) for matrices and sets. 
For any natural number $n$, we use $[n]$ to denote the set $\{1, \dots, n\}$.
For clarity when dealing with a sequence of objects, we use the superscript ${(i)}$ to denote the $i$-th object in the sequence, and subscript $j$ to denote the $j$-th entry. For example, for a sequence of vectors $\{x^{(i)}\}_{i \in [n]}$, $x^{(1)}_2$ represents the second entry of vector $x^{(1)}$.
We use $\R$ to denote the set of real numbers.
We use $\onevct$ to denote the all-one vector.
For any matrix $A$, we use $\lambda_m(A)$ to denote its $m$th-smallest eigenvalue. 
For any vector $u$, we use $\diag{u}$ to denote the diagonal matrix with $u$ in the diagonal.
For every graph mentioned in this paper, without further discussion we always assume there exists no self loops.

%%%%%%%%%%%%%%%%%

\subsection{Federated Myopic Community Detection Paradigm}
A federated myopic learning model $\mathcal{M}(n, K, \{\Omega^\kth, p^\kth, q^\kth, r^\kth\}_{k=1}^K \mid y^\ast)$ consists of $n$ nodes and $K$ clients, and every client is equipped with a tuple $\{\Omega^\kth, p^\kth, q^\kth, r^\kth\}$. In this model $y^\ast \in \{+1, -1\}^n$ is the node label vector of the network, indicating the underlying community structure. We use $V = [n]$ to denote the set of nodes.

For every single client $k \in [K]$, $\Omega^\kth \subset \abs{V} \times \abs{V}$ is the \emph{field of view (FOV)} of client $k$, $p^\kth, q^\kth$ are the local \emph{signal and noise level} parameters, and $r^\kth$ is the local $\emph{censorship}$ parameter.
We say client $k$ is \emph{myopic}, if its field of view is not equal to the complete graph spanned by $V$, i.e., $\Omega^\kth \neq \abs{V} \times \abs{V}$. 

Now, nature generates a local subgraph $G^\kth = (V^\kth, E^\kth)$ for every client $k$ using the following rule: for every viewable pair $(i,j) \in \Omega^\kth$ with $i\neq j$, connect $(i,j)$ with probability $p^\kth$ if the labels are equal, i.e., $y_i^\ast = y_j^\ast$; otherwise connect $(i,j)$ with probability $q^\kth$. After that, remove all isolated nodes from $V^\kth$, and denote $n^\kth := \abs{V^\kth}$. We call $G^\kth = (V^\kth, E^\kth)$ the local subgraph observed by client $k$.

Next, client $k$ constructs the corresponding \emph{censored evidence graph} $\tilde{G}^\kth = (V^\kth, \tilde{E}^\kth)$ from the observed subgraph $G^\kth$ as follows:
the edge set $\tilde{E}^\kth$ starts empty. For every viewable pair $(i,j) \in \Omega^\kth$ with $i\neq j$, if $(i,j)$ is an edge in $E^\kth$, add $(i,j)$ to $\tilde{E}^\kth$ with probability $1-r^\kth$; if $(i,j)$ is not an edge in $E^\kth$, add it to $\tilde{E}^\kth$ with probability $r^\kth$.
Each client sends its censored evidence graph to a central server. 

We now summarize the federated myopic learning task in this paper.
It is worth mentioning, that one can only recover the node label vector $y^\ast$ up to permutation of the groups without prior knowledge. We define the community structure matrix $Y^\ast := y^\ast y^{\ast\top} \in \{+1, -1\}^{n\times n}$. For any pair of nodes $i$ and $j$, if they are in the same community then $Y_{ij}^\ast = 1$; if they are not in the same community then $Y_{ij}^\ast = -1$.  Note that the recovery of $Y^\ast$ is equivalent to the recovery of $y^\ast$, up to permutation of the groups. 
\begin{definition}[Federated Myopic Community Detection]
\textbf{Unknown}: Community structure matrix $Y^\ast = y^\ast y^{\ast\top}$ indicating the underlying network community structure.

\textbf{Observation}: Censored evidence graphs $\{\tilde{G}^\kth\}^K_{k=1}$ sent by the clients, which are generated from the local subgraphs $\{G^\kth\}^K_{k=1}$ observed by the clients; fields of view $\{\Omega^\kth\}^K_{k=1}$ of the clients.

\textbf{Problem}: Recover the hidden community structure matrix $Y^\ast$ from the censored evidence graphs $\{\tilde{G}^\kth\}^K_{k=1}$.
\label{def:problem_detection}
\end{definition}

%%%%%%%%%%%%%%%%%%%%%%%%%%%%%%%%%%

\subsection{Central Server Algorithm}
In this section, we present a central server algorithm, which recovers the hidden community structure of a network by computing a \emph{consensus graph} and solving a semidefinite program (SDP).

\begin{algorithm}
\caption{Central Server Community Detection}
\label{alg:central} 
\textbf{Input:} Evidence graphs $\{\tilde{G}^\kth\}^K_{k=1}$, client fields of view $\{\Omega^\kth\}^K_{k=1}$ \\
\textbf{Output:} Estimated community structure matrix $\hat{Y}$  
\begin{multicols}{2} 
\begin{algorithmic}[1] % enter the algorithmic environment
    \STATE Initialize G as a weighted complete graph spanned by $V$
    \STATE Set all edge weights in $G$ to $0$
    \FOR{$k \in [K]$}
        \FOR{$(i,j) \in \Omega^\kth$}
            \IF{$(i,j) \in \tilde{G}^\kth$}
                \STATE $w(i,j) \leftarrow w(i,j) + 1$
            \ELSE
                \STATE $w(i,j) \leftarrow w(i,j) - 1$
            \ENDIF
        \ENDFOR
    \ENDFOR
    
    \STATE $W \leftarrow$ weighted adjacency matrix of $G$
    \STATE Solve the following semidefinite program
        \begin{align}
        \hat{Y} = \qquad \argmax_Y \qquad &\inprod[W]{Y} \nonumber\\
        \st \qquad & Y_{ii} = 1\,, Y \succeq 0 \,. 
        \label{alg:opt}
        \end{align}
\end{algorithmic}
\end{multicols}

\end{algorithm}

Algorithm \ref{alg:central} computes a consensus graph $G = (V, E, w)$, where $w: E \to \R$ is a weight function for edges. 
For every client $k$ and every pair $(i,j)$ in the field of view of client $k$, if it is an edge in the evidence graph, the server increments the corresponding weight $w(i,j)$ in the consensus graph by $1$; otherwise, the server decrements the weight by $1$.
We use matrix $W \in \R^{n\times n}$ to denote the weighted adjacency matrix of the consensus graph, where $W_{ij} := w(i,j)$.
The server then solves program \eqref{alg:opt} for $\hat{Y}$, the estimated community structure matrix.

\begin{claim}
Efficient exact recovery of the community structure is possible. 
Under certain topological and statistical conditions, Algorithm \ref{alg:central} recovers the true community structure matrix $Y^\ast$ perfectly and efficiently.
\label{claim:main}
\end{claim}

Figure \ref{figs:client_server} illustrates an example federated myopic learning model with $n=6$ nodes and $K=3$ clients. We show fields of view of the clients, the local subgraphs, and the consensus weighted graph from top to bottom.

%%%%%%%%%%%%%%%%%%%%%%%
\subsection{Discussion}
Here we list and discuss the assumptions that will be used in our analysis. 
For simplicity of analysis, it is assumed that the groups are balanced, i.e., $\onevct^\top y^\ast = 0$. If the groups are unbalanced, one can solve the semidefinite program in Algorithm \ref{alg:central} by adding an extra constraint. For example, suppose $\onevct^\top y^\ast = n_0$, where $n_0 \neq 0$; then one can solve the SDP with the help of an extra constraint $\inprod[Y]{\onemtx} = n_0^2$.

We also assume that for each client, the signal and noise level parameters fulfill $0 < q^\kth < p^\kth < 1$. The motivation is that nodes from the same group are more likely to be connected than those from different groups.
Similarly, we assume that the censorship parameter $r^\kth$ is in the range $(0, 0.5)$. If $r^\kth = 0$, there is no censorship at all, and if $r^\kth = 0.5$, the censored evidence graph is pure noise and provides no information.

Finally, our analysis focuses on the scenario, where every client is myopic.
Our model reduces to a weighted version of the stochastic block model (SBM), if clients are not myopic and observe the whole network. Recovery in the SBM has been studied extensively in prior literature \citep{abbe2017community}. In this paper we are interested in the signal and noise level parameters, \emph{as well as} the network FOV topology of the clients (i.e., $\{\Omega^\kth\}^K_{k=1}$). Thus we focus on the myopic regime.

\section{Novel Signed Weighted Cheeger-type Inequality}
In this section, we provide a novel Cheeger-type inequality, which relates the spectral gap of a signed weighted graph Laplacian to the \emph{signed weighted edge expansion} of the graph. 
Readers should be aware that the results in this section is general, and not limited to the particular consensus graph in our model. 

Assume $G = (V,E,w)$ with $V = [n]$ is a general weighted graph, and $w: E\to \R$ is a general weight function. We use $W$ to denote the corresponding weighted adjacency matrix. For any edge $(i,j) \in E$, we have $W_{ij} := w(i,j)$; otherwise $W_{ij} := 0$. We define the signed weight matrix $W^+, W^-$ as follows: for every entry $(i,j)$, $W_{ij}^+ = \max(W_{ij}, 0)$, and $W_{ij}^- = \min(W_{ij}, 0)$. 
We now introduce the essential graph definitions.
\begin{definition}[Boundary of a Set]
For any set $S \subset V$, denote its boundary as
\[
\partial S = \{(i,j) \mid i \in S, j \notin S\} \,.
\]
\end{definition}

\begin{definition}[Boundary Weight]
For any set $S \subset V$, let $w^+(\partial S)$ and $w^-(\partial S)$ denote its signed boundary weights, and let $w(\partial S)$ denote its boundary weight, formally defined as
\[
w^+(\partial S) = \sum_{i \in S, j \notin S} W_{ij}^+ \,, \qquad
w^-(\partial S) = \sum_{i \in S, j \notin S} W_{ij}^- \,, \qquad
w(\partial S) =\sum_{i \in S, j \notin S} W_{ij} \,.
\]
\end{definition}

\begin{definition}[Node Degree and Set Degree]
For any node $i \in V$, let $d^+(i)$ and $d^-(i)$ denote its signed node degree, and let $d(i)$ denote its node degree, formally defined as
\[
d^+(i) = \sum_{j \neq i} W_{ij}^+ \,, \qquad
d^-(i) = \sum_{j \neq i} W_{ij}^- \,, \qquad
d(i) =\sum_{j \neq i} W_{ij} \,.
\]
Similarly for any set $S \subset V$, let $d^+(S)$ and $d^-(S)$ denote its signed set degree, and let $d(S)$ denote its set degree, formally defined as
\[
d^+(S) = \sum_{i \in S} d^+(i) \,, \qquad
d^-(S) = \sum_{i \in S} d^-(i) \,, \qquad
d(S) = \sum_{i \in S} d(i) \,.
\]
\label{def:node_and_set_degree}
\end{definition}

We use the shorthand notation $d^+_{\min} := \min_{i\in V} d^+(i)$ to denote the minimum positive node degree. We can similarly define the maximum degree $d_{\max} := \max_{i\in V} d(i)$, and the minimum degree $d_{\min} := \min_{i\in V} d(i)$. 
We now define the signed weighted edge expansion, an important Cheeger-type constant that will be used in our analysis.

\begin{definition}[Signed Weighted Edge Expansion]
Given a graph $G = (V,E,w)$, for any non-empty set $S \subset V$, let $\phi^+_S$ and $\phi^-_S$ denote the signed weighted edge expansion of $S$, and let $\phi_G$ denote the signed weighted edge expansion of graph $G$, formally defined as
\[
\phi^+_S = \frac{w^+(\partial S)}{d^+(S)} \,, \qquad
\phi^-_S = w^-(\partial S) \,, \qquad
\phi_G = 
\frac{1}{2}d^+_{\min} \min_{S \subset V, d^+(S) \leq d^+(V)/2}  (\phi^+_S)^2
+ 2 \min_{S \subset V} \phi^-_S \,.
\]
\label{def:cheeger_constant}
\end{definition}

We use $L := D- W$ to denote the graph Laplacian $G$, where $D := \diag{d(1),\dots,d(n)}$ is the degree matrix. 
Note that $\onevct$ is always an eigenvector of any graph Laplacian $L$, with the corresponding eigenvalue being $0$.
We use $\lperp(L)$ to denote the smallest eigenvalue of $L$, with the corresponding eigenvector being orthogonal to $\onevct$. As a side note, for unweighted graphs and positively weighted graphs, $\lperp(L)$ is exactly equivalent to $\lambda_2(L)$, the second smallest eigenvalue of the graph Laplacian. This quantity is also called the algebraic connectivity of graph $G$.
We now present our Cheeger-type inequality for signed weighted graphs.
\begin{theorem}[Cheeger-type Inequality for Signed Weighted Graphs]
For any general signed weighted graph $G = (V,E,w)$ with graph Laplacian $L$, we have
\[
\lperp(L) \geq \phi_G \,.
\]
\label{thm:cheeger}
\end{theorem}
\section{Federated Myopic Exact Recovery}
In this section, we investigate the conditions for efficient exact recovery of the community structure in a federated myopic learning model.
We say an algorithm achieves \emph{exact recovery}, if the estimated community structure matrix $\hat{Y}$ is identical to the true matrix $Y^\ast$.
Our analysis provides provable guarantee of efficient exact recovery through Algorithm \ref{alg:central}. 

Before we proceed, we first introduce the definition of the \emph{signed consensus graph}. 
\begin{definition}[Signed Consensus Graph]
For a consensus graph $G$ with the weighted adjacency matrix $W$ as in Algorithm \ref{alg:central}, and the community structure matrix $Y^\ast$ as in Definition \ref{def:problem_detection}, we use $\bar{G}$ to denote the graph generated from the adjacency matrix $\Expect{W} \circ Y^\ast$, where $\circ$ denotes the Hadamard product. We call $\bar{G}$ the signed consensus graph of $G$.
\label{def:bar_G}
\end{definition}
Naturally, $\phi_{\bar{G}}$ is defined as the edge expansion of the signed consensus graph $\bar{G}$ as in Definition \ref{def:cheeger_constant}. 
For the readers' convenience, here we restate the optimization problem in Algorithm \ref{alg:central}:
\begin{align}
\maximize_Y \qquad &\inprod[W]{Y} \nonumber\\
\st \qquad & Y_{ii} = 1\,, Y \succeq 0 \,.
\label{opt:primal}
\end{align}
It is clear that in Algorithm \ref{alg:central}, the computation of the weighted adjacency matrix $W$ and the signed consensus graph $\bar{G}$ can be done in polynomial time in terms of $n$. 
Also note that Problem \eqref{opt:primal} is a semidefinite program. It is known that semidefinite programs are convex and can be solved efficiently in polynomial time using existing solvers \citep{boyd2004convex}. 
Thus the whole server algorithm can be run in polynomial time efficiently.

It remains to prove the correctness of Algorithm \ref{alg:central}. In other words, we want to know under what conditions, Problem \eqref{opt:primal} returns the correct community structure matrix $Y^\ast$ from input $W$, the weighted adjacency matrix of the consensus graph.
For simplicity of presentation, we define a signal coefficient $s^\kth_{ij}$ for every node pair $i, j \in V$ and every client $k\in [K]$ as follows. If pair $(i,j) \notin \Omega^\kth$, we define $s^\kth_{ij} := 0$. Otherwise, we define
$s^\kth_{ij} := p^\kth + r^\kth - 2p^\kth r^\kth$ if $y_i^\ast = y_j^\ast$, and $s^\kth_{ij} := q^\kth + r^\kth - 2q^\kth r^\kth$ if $y_i^\ast \neq y_j^\ast$.
We now present our main theorem.
\begin{theorem}[Exact Recovery of Community Structure]
For a federated myopic learning model $\mathcal{M}(n, K, \{\Omega^\kth, p^\kth, q^\kth, r^\kth\}_{k=1}^K \mid y^\ast)$ with the signed consensus graph $\bar{G}$, 
if $\phi_{\bar{G}} > 0$ and
\[
\max_{i\in V} \abs{\sum_{j\in V,k\in [K]} s^\kth_{ij} (1-s^\kth_{ij})} 
=
O\left(\frac{\phi_{\bar{G}}^2 - \phi_{\bar{G}}\log n}{\log n}\right) \,,
\]
then Problem \eqref{opt:primal} achieves exact recovery of the community structure, i.e., $\hat{Y} = Y^\ast = y^\ast y^{\ast\top}$, in polynomial time with probability tending to $1$. 
\label{thm:exact_recovery}
\end{theorem}
\section{Impossibility of Client Local Recovery}
In this section, we investigate the information-theoretic lower bounds for any algorithm to recover the local community structure from a single client's evidence graph $\tilde{G}^\kth = (V^\kth, \tilde{E}^\kth)$. 
One question may rise from Algorithm \ref{alg:central} is whether it is possible for the clients to run the semidefinite program (or any other algorithm) \emph{locally}, and obtain correct local community structure, \emph{without} sending the information to the server. This is often unwanted from a federated learning point of view, as the community structure information should be kept confidential.
As a result, we are interested in an \emph{impossibility} guarantee for any client to recover its local community structure without the assistance of a server.
Equivalently, this provides an impossibility guarantee for the server to recover the local community structure from a single client's input.

Consider client $k \in [K]$. Intuitively, for any node that is not in $V^\kth$, there is no way to recover its label. We define a subgraph community structure matrix $Y^{\kth\ast} \in \{+1, -1\}^{n^\kth \times n^\kth}$, as the matrix obtained by removing all rows and columns that are not in $V^\kth$.
As a side note, it is worth mentioning that recovery the true edge set $E^\kth$ from the censored edge set $\tilde{E}^\kth$ is very difficult, as the probability of recovering any single edge is $1-r$.

For simplicity, we introduce two client signal coefficients $s^\kth_+, s^\kth_-$, where
$s^\kth_+ := p^\kth + r^\kth - 2p^\kth r^\kth$, and $s^\kth_- := q^\kth + r^\kth - 2q^\kth r^\kth$.
We now present the impossibility theorem.
\begin{theorem}
For each client $k\in [K]$, any algorithm a learner could use to guess the local community structure $Y^{\kth\ast}$ will fail with probability at least $1/2$, if 
\[
\frac{(1-2r^\kth)^2 (p^\kth - q^\kth)^2}{\min(s^\kth_+ (1-s^\kth_+), s^\kth_- (1-s^\kth_+))} 
= 
O\left(\frac{n^\kth}{\abs{\Omega^\kth}}\right) \,.
\]
\label{thm:impossible}
\end{theorem}

\section{Illustrative Examples}
\label{section:illustrative_examples}
In this section, we discuss the relationship between the proposed signed weighted edge expansion $\phi_G$, the regular Cheeger constant $h(G)$, and the Laplacian matrix $L$, for any general graph $G$.
Our discussion is broken down into three parts: 1) $G$ as an unweighted graph; 2) $G$ as a positively weighted graph; and 3) $G$ as a signed weighted graph.

It is worth highlighting that the regular Cheeger constant $h(G)$ is only defined for unweighted graphs. 
The Laplacian matrix of any positively weighted graph is always positive semidefinite, which can be proved by invoking Gershgorin circle theorem and the diagonal dominance property. Algebraically, diagonal dominance requires that 
$\abs{d(i)} = \abs{\sum_{j\neq i} W_{ij}} \geq \sum_{j\neq i} \abs{W_{ij}} $
for every node $i$.
The direction above, however, does not hold in the case of signed weighted graphs. If there exists some $W_{ij}$ that is less than $0$, The right-hand side will be strictly greater than the left-hand side, thus breaking diagonal dominance of the Laplacian. 
To this end, the proposed signed weighted edge expansion $\phi_G$ and Theorem \ref{thm:cheeger} provides a one-way guarantee for positive semidefiniteness of the Laplacian: if $\phi_G$ is nonnegative, then the related graph Laplacian is positive semidefinite.
 
\textbf{Unweighted Graphs.} 
Suppose $G=(V,E)$ is an unweighted graph.
Recall that the definition of the regular Cheeger constant $h(S)$ of a non-empty set $S \subset V$ and the graph $G$ is defined as
$h(S) = \frac{\abs{\partial S}}{\abs{S}}$, and
$h(G) = \min_{S\subset V, \abs{S} \leq n/2} h(S)$ .
The classical Cheeger inequality states that the spectral gap $\lsec(G)$ fulfills
\begin{equation*}
\lsec(G) \geq  \frac{h(G)^2}{2 d_{\max}} \
= \min_{S\subset V, \abs{S} \leq n/2} 
\left\{\frac{\abs{\partial S}^2}{2 d_{\max} \abs{S}^2}\right\}\,.
\end{equation*}
Without loss of generality, an unweighted graph can be reduced to a general weighted graph, by setting $W_{ij} = 1$ if there exists an edge between $i$ and $j$. As a result, our signed weighted edge expansion $\phi_G$ of a unweighted graph simplifies to 
\begin{equation*}
\lsec(G) \geq \phi_G = \min_{S\subset V, \abs{S} \leq \abs{E}}
\left\{\frac{d_{\min} \abs{\partial S}^2}{2 \abs{S}^2}\right\} \,.
\end{equation*}
Here we would like to compare the lower bound provided by the classical Cheeger inequality and our novel result.
Intuitively, the term $\frac{\abs{\partial S}^2}{\abs{S}^2}$ provides a tighter bound in the classical Cheeger inequality than in ours (because of the $\abs{S} \leq n/2$ constraint instead of $\abs{S} \leq \abs{E}$), but the term $\frac{1}{d_{\max}}$ can be much worse than the term $d_{\min}$ in our bound, which is at least $1$ if the graph is connected. 

One may criticize that in our bound $d_{\min}$ could be $0$ in some cases, making the whole bound useless. We want to clarify that it is not true. If $d_{\min} = 0$, there must exist some isolated node $i$ and the graph is disconnected. In that case the edge expansion in both versions will be $0$ by picking $S = \{i\}$, thus even the classical Cheeger inequality will not provide any insight about the spectral gap.

\textbf{Positively Weighted Graphs.} 
We now assume $G = (V, E, w)$ is a positively weighted graph, with edge weight $W_{ij} \geq 0$ for every $i\neq j$.
We also assume all non-zero edge weights in $G$ are bounded between $\alpha$ and $\beta$ with $\alpha \geq \beta > 0$.
Although the regular Cheeger constant $h(G)$ is only defined for unweighted graphs, comparison with $\phi_G$ is possible by introducing the following unweighted indicator graph
$G' = (V, E')$, where $E' = \{(i,j) \mid W_{ij} > 0\}$.
In other words, an edge in $G'$ indicates a positively weighted edge in the original graph $G$. 
Then we have
$\frac{\beta^2}{\alpha} \phi_{G'}
\leq \phi_G
\leq \frac{\alpha^2}{\beta} \phi_{G'}$
for any positively weighted graph $G$. Thus, the edge expansion of a positively weighted graph is bounded by the edge expansion of the corresponding unweighted indicator graph multiplied by a constant factor (decided by the extreme weights of the graph).

%%%%%%%%%%%%%%%%%%%%%%%%%%%%%%%%%%%%%%%%%%%%
\begin{figure*}[t]
\centering
%\captionsetup[subfigure]{labelformat=empty}

\begin{subfigure}[b]{.25\linewidth}
\includegraphics[width=\linewidth]{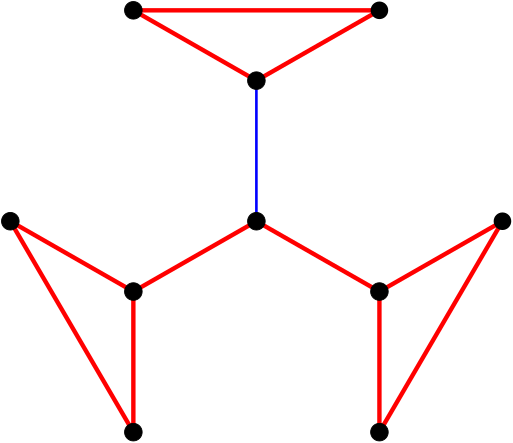}
\caption{Star graph}
\label{fig:psd_star}
\end{subfigure}
\hspace{10mm}
\begin{subfigure}[b]{.25\linewidth}
\includegraphics[width=\linewidth]{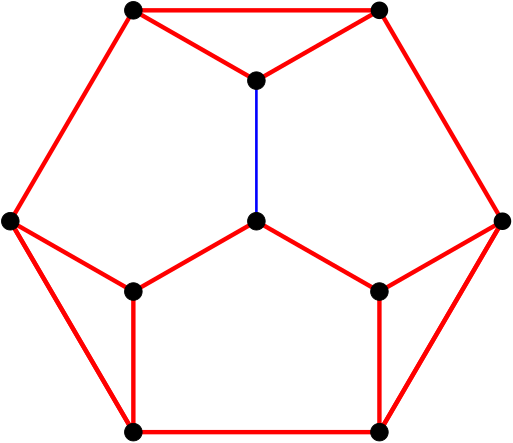}
\caption{Regular graph}
\label{fig:psd_regular}
\end{subfigure}
\hspace{10mm}
\begin{subfigure}[b]{.25\linewidth}
\centering
\includegraphics[width=0.75\linewidth]{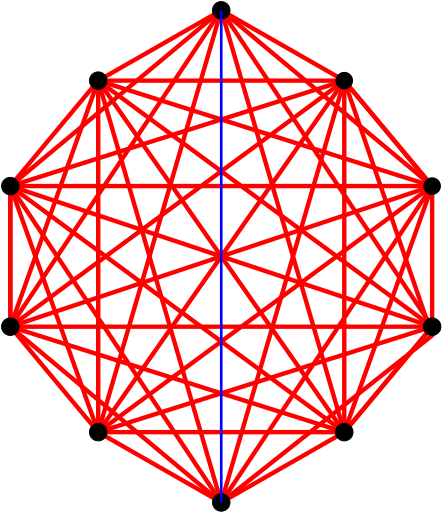}
\caption{Complete graph}
\label{fig:psd_complete}
\end{subfigure}
% \vspace{-.5mm}
\caption{Signed weighted graph with one edge perturbation. In each graph, a red edge has a positive weight of $1$. The blue edge is the the perturbed edge: its weight goes from $1$ to $-1$. The signed weighted edge expansion $\phi_G$ of every graph is recorded in Table \ref{tab:graphs}.}
\label{figs:psd}
% \vspace{-3mm}
\end{figure*}

%%%%%%%%%%%%%%%%%%%%%%%%%%%%%%%%%%%%%%%%%%%%%

\textbf{Signed Weighted Graphs.} It would be hard to make general claims about the edge expansion if the graph edge weights are allowed to be negative. Intuitively, with more negatively weighted edges, the Laplacian is more likely to break positive semidefiniteness. Thus, the negatively weighted edges can be interpreted as perturbation in the Laplacian matrix, i.e., making non-diagonal entries positive.
This gives rise to a question: what network topology is more robust to the perturbation of negatively weighted edges, without breaking positive semidefiniteness of the Laplacian?
The question is also related to the federated myopic learning model, because without positive definiteness it would be impossible to recover the network structure using the SDP approach. 

Consider the following example graphs in Figure \ref{figs:psd}. Here every graph is a signed weighted graph with $n = 10$ nodes. A red edge is assigned a positive weight of $1$ and remains constant. The blue edge is a variable edge: its weight is assigned to be $1, 0.5, 0, -0.5$ and $1$ in each iteration respectively. We check the signed weighted edge expansion of each graph and record the values in Table \ref{tab:graphs}.
One can see that in terms of the signed weighted edge expansion $\phi_G$, the complete graph is the most robust one among three, and the star graph is the least. We also test if the graph Laplacians are positive semidefinite and mark the results in bold font. The test shows that the Laplacian of a star graph is more prone to the perturbation of a negatively weighted edge without breaking PSD, while the complete graph is less prone. 

%%%%%%%%%%%%%%%%%%%%%%%%
\begin{table}[ht!]
\centering
\begin{tabular}{lccccc}
\toprule
            & \multicolumn{5}{c}{Perturbed Edge Weight}        \\ 
            \cmidrule(r){2-6}
            & 1      & 0.5    & 0      & -0.5    & -1      \\ 
\midrule
Star Graph    & 0.019 & 0.005 & 0      & \textbf{-0.167} & \textbf{-0.333} \\ 
Regular Graph & 0.167 & 0.116 & 0.074 & -0.093 & \textbf{-0.333} \\ 
Complete Graph & 1.389 & 1.334 & 1.280 & 1.180  & 1.080  \\ 
\bottomrule
\end{tabular}
\caption{Signed weighted edge expansion $\phi_G$ of every graph in Figure \ref{figs:psd} with different weights assigned to the blue edge. Graphs breaking positive semidefiniteness are marked in bold font.}
\label{tab:graphs}
\end{table}

\subsection{Server Algorithm Validation}
\label{subsection:server_experiments}

In this section, we validate the proposed Algorithm \ref{alg:central} through synthetic experiments.
In the following experiments, we fix the parameters $p, q$ and $r$ to be the same across all clients. In particular, we set $p = 0.9, q = 0.1$. The number of nodes is fixed to be $30$.
We control the field of view of each client as follows: every client randomly samples $M$ nodes, and the FOV of the client is the complete subgraph spanned by the sampled nodes.
Thus, the parameters are the size of the field of view $M$, and the number of clients $K$. These are the x-axis and y-axis in Figure \ref{figs:regimes}, respectively. 
Furthermore, we are interested in comparing two regimes. One is the \emph{multiview regime}, in which there are few clients (maximum of $20$), but every client sends an evidence graph with a high signal-to-noise ratio ($r=0.1$). The other one is the \emph{federated regime}, in which there are many clients (maximum of $200$), but every client sends a noisy evidence graph ($r=0.4$).
The experiments are run on a local machine with a Intel Core i9-10900K CPU.

\begin{figure*}[ht!]
    \centering
    %\captionsetup[subfigure]{labelformat=empty}
    
    \begin{subfigure}[b]{.45\linewidth}
    \includegraphics[width=\linewidth]{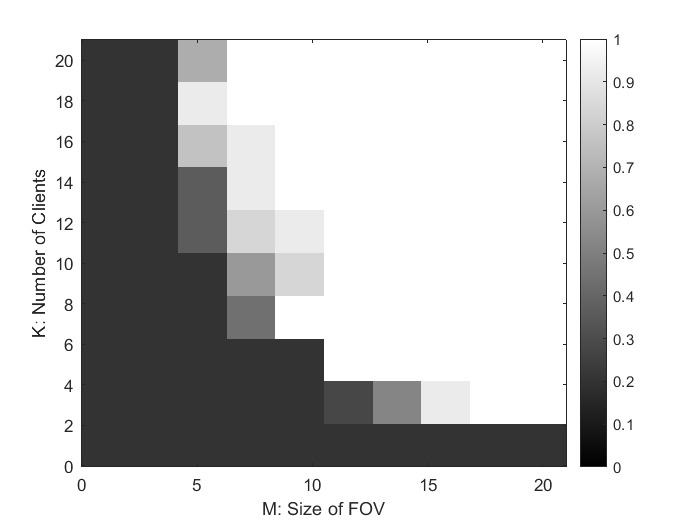}
    \caption{Multiview regime ($r = 0.1$)}
    \label{fig:r01}
    \end{subfigure}
    \hspace{10mm}
    \begin{subfigure}[b]{.45\linewidth}
    \includegraphics[width=\linewidth]{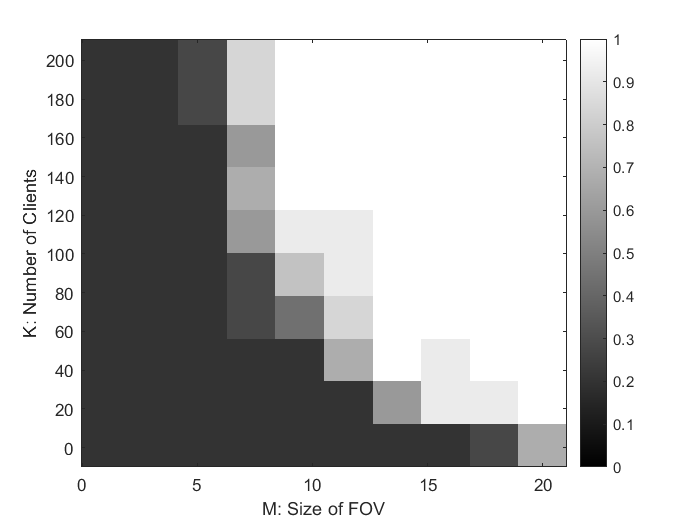}
    \caption{Federated regime ($r = 0.4$)}
    \label{fig:r04}
    \end{subfigure}
    % \vspace{-.5mm}
    \caption{Synthetic experiments to validate Algorithm \ref{alg:central}. The x-axis is the size of FOV of each client, and the y-axis is the number of clients in the federation. Even in the highly noisy case (federated regime), with a large number of clients, the central server is able to recover the community structure of the network perfectly. There is tradeoff between the amount of signal sent by each client (decided by the size of FOV, as well as the signal and noise levels), and the number of clients in the federation. }
    \label{figs:regimes}
    \end{figure*}

We run Algorithm \ref{alg:central} and check agreement between the recovered label vector $\hat{y}$, and the ground truth vector $y^\ast$. We solve the semidefinite program in Algorithm \ref{alg:central} using CVX \citep{cvx, gb08}. For each pair of parameters, we count how many times (out of $10$) the recovered label vector is identical to the ground truth. Our experiments show that exact recovery of the community structure can be achieved in both regimes, with a tradeoff between the signal-to-noise ratio and the number of clients.

%%%%%%%%%%%%%%%%%%%%%%%%%%%%%%%%%%%%%%%%%%%%

%%%%%%%%%%%%%%%%%%%%%%%%%%%%%%%%%%%%%%%%%%%%%%%%%%%%%%%%%%%

\bibliography{0_main}
\bibliographystyle{abbrvnat}

\onecolumn
\newpage
\appendix

\section{Overview}
Before going into the proof details, we first discuss the connection between the edge expansion $\phi_{\bar{G}}$, the exact recovery task (as in Theorem \ref{thm:exact_recovery}), and the types of graphs (as in Section \ref{section:illustrative_examples}).

Recall that in our federated myopic learning model, $W$ is the weighted adjacency matrix computed by the server, as in Algorithm \ref{alg:central}.
Intuitively, for every pair $i$ and $j$, $\Expect{W_{ij}}$ represents the average connection strength between these two nodes in the network. In other words, if $\Expect{W_{ij}} > 0$, node $i$ and $j$ are more likely to be in the same community; if $\Expect{W_{ij}} < 0$, they are more likely to be in different communities. 
In the proof we will show that $\sum_{k\in [K]} s_{ij}^\kth = \Expect{W_{ij}}$.
We call $(i,j)$ a ``good'' edge if $\sum_{k\in [K]} s_{ij}^\kth y^\ast_i y^\ast_j = \Expect{W_{ij}} y^\ast_i y^\ast_j > 0$, and a ``bad'' edge if $\sum_{k\in [K]} s_{ij}^\kth y^\ast_i y^\ast_j = \Expect{W_{ij}} y^\ast_i y^\ast_j  < 0$.
By Definition \ref{def:bar_G}, the adjacency matrix of $\bar{G}$ is $\Expect{W} \circ Y^\ast$.
As a consequence, the signed consensus graph $\bar{G}$ and the corresponding edge expansion $\phi_{\bar{G}}$ are heavily related to the chance of successful community detection.
In the proof we use $L^\ast$ to denote the graph Laplacian of $\bar{G}$. We show that successful exact recovery requires that the second smallest eigenvalue of $L^\ast$ is strictly positive. By Theorem \ref{thm:cheeger}, this is fulfilled as long as the edge expansion $\phi_{\bar{G}}$ is strictly positive. 

Now we discuss two types of graphs characterized by the signed consensus graph $\bar{G}$ and the signed weight matrix $\Expect{W} \circ Y^\ast$.

\textbf{Positively Weighted Graphs.} In this case, every entry in $\Expect{W} \circ Y^\ast$ is greater than or equal to $0$. This implies that every edge in the graph is a ``good'' edge. Community detection is easy in this case, as long as the graph is connected. 
In fact, exact recovery in this case can be achieved through a greedy algorithm: the algorithm adds the first node to either community, and then adds the neighboring nodes to the same community if they are connected by a positive edge $W_{ij} > 0$, or adds them to the other community if connected by a negative edge $W_{ij} < 0$. The algorithm repeats until all nodes are assigned to a community. The edge expansion $\phi_{\bar{G}}$ is strictly positive in this case.

\textbf{Signed Weighted Graphs.} In this case, every entry in $\Expect{W} \circ Y^\ast$ is greater than, less than, or equal to $0$. Community detection is harder, since ``bad'' edges exist in this case. One can see that the simple greedy algorithm described above will not be able to find consistent assignments because of the ``bad'' edges. Just like the experiments in Figure \ref{figs:psd} and Table \ref{tab:graphs}, the chance of successful community detection depends on how many negative entries exist in $\Expect{W} \circ Y^\ast$ and how negative the entries are: recovery becomes hard if there are many ``bad'' edges, or some edges are ``really bad.'' The edge expansion $\phi_{\bar{G}}$ can be positive or negative in this case.

In the proof of Theorem \ref{thm:exact_recovery}, we will characterize the relationship between the chance of exact recovery and the edge expansion $\phi_{\bar{G}}$ in a more rigorous way.

%%%%%%%%%%%%%%%%%%%%%%%%%%%%%%%%%%%%%%%%%%%%%%%%%%%%%%%%%%%%%%%%%%%%%%%%%%%%%%%%%%%%%%%%%%%%%%%%%%%%%%%%%%%

\section{Proof of Theorem \ref{thm:cheeger}}

In this section, we provide the proof of the Cheeger-type inequality for signed weighted graphs as stated in Theorem \ref{thm:cheeger}. 
Recall that $G=(V,E,w)$ is a general weighted graph, $W$ is the weighted adjacency matrix, $D$ is the degree matrix, and $L = D-W$ is the graph Laplacian. 

First we introduce the following definitions and lemmas that will be used later in the proof.

\begin{definition}[Rayleigh Quotient]
Let $L = D-W$ be an $n \times n$ graph Laplacian matrix, where $D$ is the degree matrix, and $W$ is the weighted adjacency matrix. The \emph{Rayleigh Quotient} of a vector $v\in R^n$ with respect to $L$ is defined as 
\[
R_L(v) := \frac{v^\top L v}{v^\top v} = \frac{\sum_{i<j} W_{ij} (v_i - v_j )^2}{v^\top v} \,.
\]
Similarly, the signed Rayleigh Quotients are defined as
\[
R_L^+ (v) := \frac{\sum_{i<j} W_{ij}^+ (v_i - v_j )^2}{v^\top v} \,,
\qquad
R_L^- (v) := \frac{\sum_{i<j} W_{ij}^- (v_i - v_j )^2}{v^\top v}  \,.
\]
Note that $R_L(v) = R_L^+ (v) + R_L^- (v)$.
\end{definition}

Using the variational characterization of eigenvalues, it follows that $\lperp(L) = \min_{v\perp \onevct} R_L(v)$.

\begin{lemma}
For any non-zero $\alpha \in \R$, it follows that
\[
R_{L} (v) 
=
R_{L} (\alpha v) \,.
\]
\label{lemma:RLscaling}
\end{lemma}
\begin{proof}
Note that
$
R_{L} (\alpha v)
= \frac{(\alpha v )^\top L (\alpha v)}{(\alpha v )^\top (\alpha v)} 
= R_{L} (v) 
$.
\end{proof}

\begin{lemma}
For any $\delta \in \R$ and $v\in\R^n, v\perp \onevct$, it follows that
\[
R_{L}^+ (v) 
\geq 
R_{L}^+ ( v + \delta \onevct) \,.
\]
\label{lemma:RLshifting}
\end{lemma}
\begin{proof}
Starting from the right-hand side, we have
\begin{align*}
R_{L}^+ ( v + \delta \onevct)
&= \frac{\sum_{i<j} W_{ij}^+ \left(( v_i + \delta)-( v_j + \delta)\right)^2}{\sum_{i}  \left( v_i + \delta \right)^2} \\
&= \frac{ \sum_{i<j} W_{ij}^+ \left(v_i - v_j \right)^2}{\sum_{i}  \left( v_i + \delta \right)^2} \\
&= \frac{ \sum_{i<j} W_{ij}^+ \left(v_i - v_j \right)^2}{\sum_{i}  \left(  v_i^2 + \delta^2 + 2\delta v_i \right)} \\
&= \frac{ \sum_{i<j} W_{ij}^+ \left(v_i - v_j \right)^2}{ \sum_{i} v_i^2 + n\delta^2 + 2\delta \sum_i v_i } \\
&= \frac{ \sum_{i<j} W_{ij}^+ \left(v_i - v_j \right)^2}{ \sum_{i} v_i^2 + n\delta^2} \\
&\leq \frac{ \sum_{i<j} W_{ij}^+ \left(v_i - v_j \right)^2}{ \sum_{i} v_i^2} \\
&= R_{L}^+ (v) \,.
\end{align*}
\end{proof}

We now proceed to prove the main theorem. 
Our proof takes a constructive approach: if we can construct any witness set $S, S' \subset V$ with $d^+(S) \leq d^+(V)/2$ fulfilling
$ \lperp(L) \geq \frac{1}{2}d^+_{\min} (\phi^+_S)^2
+ 2 \phi^-_{S'}$, then it follows from Definition \ref{def:cheeger_constant} that $\lperp(L) \geq \phi_G = 
\frac{1}{2}d^+_{\min} \min_{S \subset V, d^+(S) \leq d^+(V)/2}  (\phi^+_{S})^2
+ 2 \min_{S \subset V} \phi^-_S$.

\begin{proof}[Proof of Theorem \ref{thm:cheeger}]
Suppose $v$ is the eigenvector associated with the eigenvalue $\lperp(L)$. By definition, it follows that $v \perp \onevct$. By Lemma \ref{lemma:RLscaling}, without loss of generality we assume that the eigenvector $v$ has unit norm, i.e., $\norm{v} = 1$. We also assume that $v$ is sorted in ascending order, i.e., $v_1 \leq \dots \leq v_n$. 

Let $m$ be the smallest integer, such that 
$\sum_{i=1}^{m+1} d^+(i) > \frac{1}{2} d^+(V)$.
Then $v - v_m \onevct$ is centered at $m$. We can find $\alpha, \alpha' > 0$, such that 
\[
\alpha^2 (v_1-v_m)^2 + \alpha^2 (v_n-v_m)^2 = 1 \,, \qquad
\alpha'^2 v_1^2 + \alpha'^2 v_n^2 = 1 \,.
\]
Denote $u = \alpha(v - v_m\onevct)$, and $u' = \alpha' v$. By Lemma \ref{lemma:RLscaling} and \ref{lemma:RLshifting}, it follows that $R_L^+(u) \leq R_L^+(v)$, and $R_L^-(u') = R_L^-(v)$.
Our goal is to prove
\begin{equation}    
\lperp(L) = R_L(v) = R_L^+(v) + R_L^-(v) 
\geq R_L^+(u) + R_L^-(u')
\geq \frac{1}{2}d^+_{\min} (\phi^+_S)^2
+ 2 \phi^-_{S'} 
\,.
\label{eq:cheeger_ineq_derivation}
\end{equation}
for some $S, S' \subset V$ with $d^+(S) \leq d^+(V)/2$.

We first prove the positive part $R_L^+(u) \geq \frac{1}{2}d^+_{\min} (\phi^+_S)^2$ of \eqref{eq:cheeger_ineq_derivation}. The proof of the negative part will be similar. 

We define a random variable $t$ on the support $[u_1, u_n]$, with probability density function $f(t) = 2\abs{t}$. One can verify that $\int_{t = u_1}^{u_n} 2\abs{t} = 1$ because $u_1^2 + u_n^2 = 1$, thus $f(t)$ is a valid probability density function.
Then, for any interval $[a,b]$, it follows that the probability of $t$ falling in the interval is 
\[
\Prob{a \leq t \leq b}
=
\int_{t=a}^b 2\abs{t} = - a^2\sgn{a} + b^2\sgn{b} 
 \,.
\]
It can be verified that
\begin{equation}
(a-b)^2 / 2
\leq 
\Prob{a \leq t \leq b}
\leq
\abs{a-b} (\abs{a}+\abs{b}) \,,
\label{eq:cheeger_set_prob_bound}
\end{equation}
in the range $[-1,1]$.
We construct a random set $S_t := \{i \mid u_i \leq t\}$. 
Note that we have $\sum_{i=1}^m d^+(i) \geq \frac{1}{2} d^+(V)$, and $u_m = 0$. As a result, if $t \leq 0$, we have $\min(d^+({S_t}), d^+({V\setminus S_t})) = d^+({S_t})$, otherwise $\min(d^+({S_t}), d^+({V\setminus S_t})) = d^+({V\setminus S_t})$. By Definition \ref{def:node_and_set_degree}, we obtain
\[
\Expect{d^+ (S_t)}
= \sum_i \Prob{u_i \leq t} d^+(i)  \,, \qquad
\Expect{d^+ (V \setminus S_t)}
= \sum_i \Prob{u_i \geq t} d^+(i)  \,.
\]
It follows that
\begin{align*}
\Expect{\min(d^+({S_t}), d^+({V\setminus S_t}))}
&= \sum_{i \leq m} \Prob{u_i \leq t \leq 0} d^+(i) 
+ \sum_{i > m} \Prob{u_i \geq t >  0} d^+(i)  \\
&= \sum_i u_i^2 d^+(i) \,.
\end{align*}

We now analyze the positive boundary weight of $S_t$. It follows that
\begin{align*}
\Expect{w^+ (\partial S_t)} 
&= \Expect{\sum_{i\in S_t, j\notin S_t} W_{ij}^+} \\
&= \sum_{i<j} \Prob{u_i \leq t \leq u_j} W_{ij}^+ \\
&\leq \sum_{i<j} \abs{u_i-u_j} (\abs{u_i}+\abs{u_j})   W_{ij}^+ \\
&\overset{(a)}{\leq} \sqrt{\sum_{i<j} (u_i-u_j)^2 W_{ij}^+} 
\sqrt{\sum_{i<j}(\abs{u_i}+\abs{u_j})^2 W_{ij}^+} \\
&= \sqrt{R_L^+(u) \sum_{i} u_i^2} 
\sqrt{\sum_{i<j}(\abs{u_i}+\abs{u_j})^2 W_{ij}^+} \\
&\overset{(b)}{\leq} \sqrt{\frac{R_L^+(u)}{d^+_{\min}} \sum_{i} u_i^2 d^+(i)} 
    \sqrt{2\sum_{i} u_i^2 d^+(i)}  \\
&= \sqrt{\frac{2R_L^+(u)}{d^+_{\min}}  }  \cdot \Expect{\min(d^+({S_t}), d^+({V\setminus S_t}))}\,,
\end{align*}
where (a) follows from the Cauchy-Schwarz inequality, and (b) follows from the definition of node degree and picking the minimum positive node degree.
Rearranging the inequality above gives us
\[
\Expect{\sqrt{\frac{2R_L^+(u)}{d^+_{\min}}} \min(d^+({S_t}), d^+({V\setminus S_t})) - w^+ (\partial S_t)} \geq 0 \,.
\]
Note that inside the expectation is a function of $t$.
Thus, there exists some $t$ fulfilling
\[
\sqrt{\frac{2R_L^+(u)}{d^+_{\min}}} \min(d^+({S_t}), d^+({V\setminus S_t})) - w^+ (\partial S_t) \geq 0\,,
\]
which leads to
\begin{equation}
R_L^+(u) \geq \frac{1}{2}d^+_{\min} (\phi^+_{S_t})^2 = \frac{1}{2}d^+_{\min} \left(\frac{w^+ (\partial S_t)}{d^+(S_t)} \right)^2    \,,
\label{eq:cheeger_rlplus}
\end{equation}
under the condition of $d^+(S_t) \leq d^+(V)/2$.

The proof of the negative part of \eqref{eq:cheeger_ineq_derivation} is similar. Instead of using the random variable $t$ on the support $[u_1, u_n]$, we define a new random variable $t'$ on the support $[u'_1, u'_n]$. We use the same probability density function $f(t') = 2\abs{t'}$, so \eqref{eq:cheeger_set_prob_bound} still holds for $t'$. 
For the negative boundary weight, we have
\begin{align*}
\Expect{w^- (\partial S_{t'})} 
&= \Expect{\sum_{i\in S_{t'}, j\notin S_{t'}} W_{ij}^-} \\
&= \sum_{i<j} \Prob{u'_i \leq t' \leq u'_j} W_{ij}^- \\
&\overset{(a)}{\leq} \frac{1}{2}\sum_{i<j} (u'_i - u'_j)^2    W_{ij}^- \\
&= \frac{1}{2}R_L^-(u') \sum_i {u'}_i^{2} \\
&\overset{(b)}{\leq} \frac{1}{2}R_L^-(u') \,.
\end{align*}
where (a) follows from the left-hand side of  \eqref{eq:cheeger_set_prob_bound}, and (b) follows from the fact that $\sum_i {u'}_i^{2} \geq {u'}_1^{2} + {u'}_n^{2} = 1$.
Thus, there exists some $t'$ fulfilling
\begin{equation}
R_L^-(u') \geq 2\phi^-_{S_{t'}} = 2 w^- (\partial S_{t'})      \,.
\label{eq:cheeger_rlminus}
\end{equation}

Combining \eqref{eq:cheeger_rlplus} and \eqref{eq:cheeger_rlminus} completes our proof of \eqref{eq:cheeger_ineq_derivation}.
\end{proof}

%%%%%%%%%%%%%%%%%%%%%%%%%%%%%%%%%%%%%%%%%%%%%%%%%%%%%%%%%%%%%%%%%%%%%%%%%%%%%%%%%%%%%%%%%%%%%%%%%%%%%%%%%%%

\section{Proof of Theorem \ref{thm:exact_recovery}}
For every client $k$ with evidence graph $\tilde{G}^\kth$, we introduce a weighted matrix $W^\kth$ defined as follows. For every pair $(i,j)$, if $(i,j) \notin \Omega^\kth$ we assign $W^\kth_{ij} = 0$. If $(i,j) \in \Omega^\kth$ and $(i,j)$ is an edge in $\tilde{G}^\kth$, we assign $W^\kth_{ij} = 1$; otherwise we assign $W^\kth_{ij} = -1$.
From Algorithm \ref{alg:central}, one can see that the server weighted adjacency matrix is the summation of the clients' weighted matrices, i.e., 
$
W = \sum_{k\in [K]} W^\kth
$.

We take a primal-dual witness approach to show that the SDP \eqref{opt:primal} achieves exact recovery. 
For the readers' convenience, here we restate the SDP \eqref{opt:primal}:
\begin{align*}
\maximize_Y \qquad &\inprod[W]{Y} \nonumber\\
\st \qquad & Y_{ii} = 1\,, Y \succeq 0 \,.
\end{align*}
Suppose the Lagrangian dual variables are $\Lambda$ (for $Y_{ii} = 1$) and $B$ (for $Y \succeq 0$), respectively, where $\Lambda$ is a diagonal matrix of size $\R^{n\times n}$, and $B \succeq 0$ is a PSD matrix of the same size.
The Lagrangian dual problem of \eqref{opt:primal} is
\begin{align}
\minimize_\Lambda \qquad &\tr{\Lambda} \nonumber\\
\st \qquad & \Lambda \text{ is diagonal}\,, \Lambda-W \succeq 0 \,.
\label{opt:dual}
\end{align}

We now list the Karush–Kuhn–Tucker (KKT) conditions for a primal and dual pair $(Y, \Lambda, B)$ to be optimal.
\begin{align}
\Lambda - W - B &= 0 \,, \tag{Stationarity} \label{kkt:stationarity} \\
Y_{ii} = 1\,, \quad Y &\succeq 0 \,, \tag{Primal Feasibility} \label{kkt:pf} \\
\Lambda \text{ is diagonal}\,, \quad B &\succeq 0\,, \tag{Dual Feasibility} \label{kkt:df} \\
\inprod[B]{Y} &= 0 \,. \tag{Complementary Slackness} \label{kkt:cs}
\end{align}

To guarantee $Y^\ast = \yyast$ is an optimal solution to the SDP \eqref{opt:primal}, all KKT conditions need to be fulfilled.
First note that $Y^\ast$ fulfills \eqref{kkt:pf} trivially.
Next, combining \eqref{kkt:stationarity} and \eqref{kkt:cs}, we obtain that an optimal solution must fulfill
\begin{equation}
\inprod[\Lambda -W]{Y^\ast} = 0 \,.
\label{opt:inprodzero}
\end{equation}
To fulfill \eqref{opt:inprodzero}, we can construct the dual variables $\Lambda^\ast$ and $B$ as follows: $\Lambda^\ast_{ii} := \sum_{j\neq i} W_{ij} y^\ast_i y^\ast_j$ for every $i\in [n]$, and $B^\ast := \Lambda^\ast - W$.
Then it only remains to prove that our construction $(Y^\ast, \Lambda^\ast, B^\ast)$ fulfills \eqref{kkt:df} and \eqref{kkt:cs}, i.e.,
\begin{equation*}
B^\ast = \Lambda^\ast - W \succeq 0\,.
\end{equation*}
One can verify that by our constriction, $y^\ast$ is always an eigenvector of $\Lambda^\ast - W$ with the corresponding eigenvalue of $0$. Thus the inequality above is equivalent to 
\begin{equation}
\lsec(\Lambda^\ast - W) \geq 0\,,
\label{opt:dual_feasibility}
\end{equation}

The KKT conditions, once fulfilled, guarantee that $Y^\ast = \yyast$ is an optimal solution to the primal problem. 
However there could exist other sets of primal and dual variables satisfy all KKT conditions above. To illustrate this, we construct a set of example primal and dual variables $(\tilde{Y}, \tilde{\Lambda}, \tilde{B})$ as follows: $\tilde{Y} := \onemtx$ is the all-one matrix, $\tilde{\Lambda}_{ii} := \sum_{j\neq i} W_{ij}$ for every $i\in [n]$, and $\tilde{B} := \tilde{\Lambda} - W$.
One can verify that $(\tilde{Y}, \tilde{\Lambda}, \tilde{B})$ fulfill all KKT conditions above, and as a result, $\tilde{Y} = \onemtx$ is an optimal solution to the primal SDP.

To ensure that $Y^\ast = \yyast$ is the \emph{unique} optimal solution to \eqref{opt:primal} and eliminate all other undesirable solutions, we present the following lemma about uniqueness.

\begin{lemma}[Uniqueness Condition]
The SDP \eqref{opt:primal} achieves exact recovery and returns the unique optimal solution $Y^\ast = \yyast$, if 
\[
\lsec(\Lambda^\ast - W) > 0 \,,
\]
where $\Lambda^\ast$ is a diagonal matrix, such that $\Lambda^\ast_{ii} = \sum_{j\neq i} W_{ij} y^\ast_i y^\ast_j$ for every $i\in [n]$.
\label{lemma:opt_uniqueness}
\end{lemma}
\begin{proof}
First note that $(Y^\ast, \Lambda^\ast, B^\ast)$ constructed above fulfill all KKT conditions given that $\lsec(\Lambda^\ast - W) \geq 0$. Thus $Y^\ast = \yyast$ is an optimal solution to the SDP \eqref{opt:primal}.

Now, by enforcing the strict inequality $\lsec(\Lambda^\ast - W) > 0$, we obtain that for every vector $v$ that is not a multiple of $y^\ast$, we have $\inprod[\Lambda^\ast - W]{vv^\top} > 0$ strictly.
Thus to fulfill the optimality condition \eqref{opt:inprodzero}, the solution $Y$ must be a multiple of $Y^\ast$. Furthermore, given the constraint $Y_{ii} = 1$ \eqref{kkt:pf}, the only possible optimal solution that fulfills all KKT conditions is $Y^\ast$.
\end{proof} 

%%%%%%%%%%%%%%%%%%%%%%%%%%%%%%%%%

We now begin to prove the main theorem.
\begin{proof}[Proof of Theorem \ref{thm:exact_recovery}]
We first investigate the value of $W_{ij}^\kth$. If $(i,j)$ is not observed by client $k$, the value of $W_{ij}^\kth$ is $0$. Otherwise, if $Y_{ij}^\ast = 1$, we obtain that $W_{ij}^\kth = 1$ with probability $p^\kth (1-r^\kth) + (1-p^\kth) r^\kth = s_{ij}^\kth$, and $W_{ij}^\kth = -1$ with probability $1- s_{ij}^\kth$.
Similarly, if $Y_{ij}^\ast = -1$, we obtain that $W_{ij}^\kth = 1$ with probability $q^\kth (1-r^\kth) + (1-q^\kth) r^\kth = s_{ij}^\kth$, and $W_{ij}^\kth = -1$ with probability $1- s_{ij}^\kth$.

Our goal is to prove $\lsec(\Lambda^\ast - W) > 0$. Note that 
\begin{align}
\quad\lsec(\Lambda^\ast - W)  
= \lsec(\Lambda^\ast - W - \Expect{\Lambda^\ast - W} + \Expect{\Lambda^\ast - W}) &> 0 \nonumber \\
\Leftarrow \lsec(\Lambda^\ast -\Expect{\Lambda^\ast}) + \lsec(\Expect{W} - W) + \lsec(\Expect{\Lambda^\ast - W}) &> 0 \nonumber \\
\Leftarrow \lambda_1(\Lambda^\ast -\Expect{\Lambda^\ast}) + \lambda_1(\Expect{W} - W) + \lsec(\Expect{\Lambda^\ast - W}) &> 0 \,.
\label{eq:three_terms}
\end{align}
Thus it is sufficient to prove \eqref{eq:three_terms} holds with high probability.

First we bound the third term in \eqref{eq:three_terms}. We introduce a matrix $L^\ast$ defined as $L^\ast := (\Lambda^\ast - W) \circ Y^\ast$, where $\circ$ denotes the Hadamard product. It is worth mentioning that $\Lambda^\ast - W$ and $L^\ast$ share the same eigenvalues, since 
$L^\ast
=(\Lambda^\ast - W) \circ Y^\ast 
= \Lambda^\ast \circ Y^\ast - W \circ Y^\ast
= \Lambda^\ast - W \circ Y^\ast$.
Furthermore, $L^\ast$ satisfies the definition of graph Laplacians, such that 
$L_{ii} = \sum_{j\neq i} L_{ij} = \sum_{j\neq i} W_{ij} Y^\ast_{ij}$ 
for every $i\in V$, and 
$L_{ij} = - W_{ij} Y^\ast_{ij}$ 
for every $i\neq j$.
Then by Definition \ref{def:bar_G}, one can see that $\Expect{L^\ast}$ is the Laplacian of graph $\bar{G}$, with the weighted adjacency matrix $\Expect{W} \circ Y^\ast$.
By Theorem \ref{thm:cheeger}, it follows that as long as $\phi_{\bar{G}} > 0$, we have $\lperp(\Expect{L^\ast}) \geq \phi_{\bar{G}} > 0$.
Since $\onevct$ is always an eigenvector of $\Expect{L^\ast}$ with the eigenvalue $0$, this implies that all other eigenvalues of $\Expect{L^\ast}$ are strictly positive, which leads to the conclusion that 
\[
\lsec(\Expect{L^\ast}) = \lperp(\Expect{L^\ast}) \geq \phi_{\bar{G}} > 0 \,.
\]
Since  $\Lambda^\ast - W$ and $L^\ast$ share the same eigenvalues, we obtain 
\begin{equation}
\lsec(\Expect{\Lambda^\ast - W}) > \phi_{\bar{G}} \,.
\label{eq:expectation_bound}
\end{equation}

Next we bound the first term in \eqref{eq:three_terms}. For every $i, j \in V, k \in [K]$, we construct a matrix $\Phi^{(i,j,k)} \in \R^{n\times n}$ as follows: 
\[
\Phi^{(i,j,k)}_{ii} := W_{ij}^\kth Y^\ast_{ij} \,, \qquad
\Phi^{(i,j,k)}_{jj} := W_{ij}^\kth Y^\ast_{ij} \,,
\]
and all other entries are $0$. 
Then it follows that 
\[
\Lambda^\ast -\Expect{\Lambda^\ast}
= \sum_{i,j\in V, k\in [K]} 
\Phi^{(i,j,k)} - \Expect{\Phi^{(i,j,k)}} \,.
\]
To bound the minimum eigenvalue of the expression above, we use the matrix Bernstein inequality \citep{tropp2012user}. 
Note that each matrix $\Phi^{(i,j,k)} - \Expect{\Phi^{(i,j,k)}}$ is independent, has zero mean, and the maximum eigenvalue is bounded above by $1$. 
Also, since $W_{ij}^\kth$ is a $+1$/$-1$ random variable with parameter $s_{ij}^\kth$, we obtain that 
\[
\norm{\sum_{i,j\in V, k\in [K]} 
\Expect{\left(\Phi^{(i,j,k)} - \Expect{\Phi^{(i,j,k)}}\right)^2}}
\leq 
4 \max_{i\in V} \abs{\sum_{j\in V,k\in [K]} s^\kth_{ij} (1-s^\kth_{ij})} \,.
\]
We denote the right-hand side expression above by $\sigma_{\Phi}^2$. Applying the matrix Bernstein inequality, we obtain that for every $t >0$, 
\begin{align*}
\Prob{\lambda_1(\Lambda^\ast - \Expect{\Lambda^\ast}) \leq -t}
&\leq n \exp\left(\frac{-t^2/2}{\sigma_{\Phi}^2+t/3}\right) \,.
\end{align*}
By setting $t = \phi_{\bar{G}} / 2$ and requiring the probability to be bounded above by $n^{-1}$, we obtain that
\begin{equation}
\sigma_\Phi^2 \leq \frac{3\phi_{\bar{G}}^2 - 8\phi_{\bar{G}} \log n}{48\log n} \,.
\label{eq:first_term_bound}
\end{equation}

Finally we bound the second term in \eqref{eq:three_terms}. For every $i, j \in V, k \in [K]$, we construct a matrix $\Psi^{(i,j,k)} \in \R^{n\times n}$ as follows: 
\[
\Psi^{(i,j,k)}_{ij} := W_{ij}^\kth Y^\ast_{ij} \,, \qquad
\Psi^{(i,j,k)}_{ji} := W_{ij}^\kth Y^\ast_{ij} \,,
\]
and all other entries are $0$. 
Then it follows that 
\[
\Expect{W} - W
= \sum_{i,j\in V, k\in [K]} 
\Expect{\Psi^{(i,j,k)}} - \Psi^{(i,j,k)} \,.
\]
To bound the minimum eigenvalue of the expression above, we use the matrix Bernstein inequality again. 
Similar to the previous case, each matrix $ \Expect{\Psi^{(i,j,k)}} - \Psi^{(i,j,k)}$ is independent, has zero mean, and the maximum eigenvalue is bounded above by $1$. 
Also, since $W_{ij}^\kth$ is a $+1$/$-1$ random variable with parameter $s_{ij}^\kth$, we obtain that 
\[
\norm{\sum_{i,j\in V, k\in [K]} 
\Expect{\left(\Expect{\Psi^{(i,j,k)}} - \Psi^{(i,j,k)}\right)^2}}
\leq 
4 \max_{i\in V} \abs{\sum_{j\in V,k\in [K]} s^\kth_{ij} (1-s^\kth_{ij})} \,.
\]
Note that the right-hand side expression above is exactly the same as the one in the previous case. 
Thus by invoking the matrix Bernstein inequality, setting $t = \phi_{\bar{G}} / 2$, and requiring the probability to be bounded above by $n^{-1}$, we obtain the same condition
\begin{equation}
\sigma_\Phi^2 \leq \frac{3\phi_{\bar{G}}^2 - 8\phi_{\bar{G}} \log n}{48\log n} \,.
\label{eq:second_term_bound}
\end{equation}

Combining the results of \eqref{eq:expectation_bound}, \eqref{eq:first_term_bound}, and \eqref{eq:second_term_bound}, we obtain that $\lsec(\Lambda^\ast - W) > 0$ holds with probability at least $1-2n^{-1}$, as long as 
\[
4 \max_{i\in V} \abs{\sum_{j\in V,k\in [K]} s^\kth_{ij} (1-s^\kth_{ij})} 
\leq 
\frac{3\phi_{\bar{G}}^2 - 8\phi_{\bar{G}} \log n}{48\log n} \,.
\]
This completes our proof.
\end{proof}

%%%%%%%%%%%%%%%%%%%%%%%%%%%%%%%%%%%%%%%%%%%%%%%%%%%%%%%%%%%%%%%%%%%%%%%%%%%%%%%%%%%%%%%%%%%%%%%%%%%%%%%%%%%

\section{Proof of Theorem \ref{thm:impossible}}

\begin{proof}[Proof of Theorem \ref{thm:impossible}]
We use $\mathcal{Y}^\kth$ to denote the hypothesis class of the local community structure of client $k$. One can see that the size of the hypothesis space is
$
\abs{\mathcal{Y}^\kth} = 2^{n^\kth}
$.
By Fano's inequality \citep{cover1999elements}  for any estimator $\hat{Y}^\kth$, it follows that
\begin{align*}
\Prob{\hat{Y}^\kth \neq Y^{\kth \ast}} 
&\geq 1 - \frac{I(Y^{\kth \ast}, W^\kth) + \log 2}{\log \mathcal{Y}^\kth}
= 1 - \frac{I(Y^{\kth \ast}, W^\kth) + \log 2}{n^\kth \log 2} \,,
\end{align*}
where $I(Y^{\kth \ast}, W^\kth)$ denotes the mutual information between the two matrices. 
We use $\mathbb{KL}(P_1 \Vert P_2)$ to denote the Kullback–Leibler divergence between two distributions $P_1$ and $P_2$.
Using a pairwise KL-based bound \citep[p. 428]{yu1997assouad} we obtain
\begin{align*}
I(Y^{\kth\ast}, W^\kth)
&\leq \frac{1}{\abs{\mathcal{Y}^\kth}^2} \sum_{Y^\kth, Y^{\kth '} \in \mathcal{Y}} \mathbb{KL} (P_{W^\kth \mid Y^\kth} \Vert P_{W^\kth \mid Y^{\kth '}})\\
&\leq \max_{Y^\kth, Y^{\kth '} \in \mathcal{Y}} \mathbb{KL} (P_{W^\kth \mid Y^\kth} \Vert P_{W^\kth \mid Y^{\kth '}})\\
&= \max_{Y^\kth, Y^{\kth '} \in \mathcal{Y}} \mathbb{KL} (P_{W^\kth \mid Y^\kth} \Vert P_{W^\kth \mid Y^{\kth '}})\\
&= \max_{Y^\kth, Y^{\kth '} \in \mathcal{Y}} \sum_{W^\kth} P_{W^\kth \mid Y^\kth} \log \frac{P_{W^\kth \mid Y^{\kth }}}{P_{W^\kth \mid Y^{\kth '}}}\\
&\overset{(a)}{\leq} \abs{\Omega^\kth} \cdot \max_{Y_{ij}, Y_{ij}'} \sum_{W_{ij}^\kth} P_{W_{ij}^\kth \mid Y_{ij}^\kth} \log \frac{P_{W^\kth \mid Y_{ij}^{\kth }}}{P_{W^\kth \mid Y_{ij}^{\kth '}}}\\
&\overset{(b)}{\leq} \abs{\Omega^\kth} \cdot \max\left(s^\kth_+ \log \frac{s^\kth_+}{s^\kth_-} + (1-s^\kth_+) \log \frac{1-s^\kth_+}{1-s^\kth_-},
s^\kth_- \log \frac{s^\kth_-}{s^\kth_+} + (1-s^\kth_-) \log \frac{1-s^\kth_-}{1-s^\kth_+}\right) \\
&= \abs{\Omega^\kth} \cdot \max\left(\mathbb{KL}(s^\kth_+ \Vert s^\kth_-), \mathbb{KL}(s^\kth_- \Vert s^\kth_+)\right) \,,
\end{align*}
where $P_{W^\kth \mid Y^\kth}$ denotes the distribution of $W^\kth$ given $Y^\kth$, and (a) follows from an entrywise decomposition. (b) holds by the fact that given $Y_{ij} = 1$, $W_{ij}^\kth = 1$ with probability $p^\kth (1-r^\kth) + (1-p^\kth) r^\kth = s_{+}^\kth$, and $W_{ij}^\kth = -1$ with probability $1- s_{+}^\kth$;
similarly, given $Y_{ij}^\ast = -1$, $W_{ij}^\kth = 1$ with probability $q^\kth (1-r^\kth) + (1-q^\kth) r^\kth = s_{ij}^\kth$, and $W_{ij}^\kth = -1$ with probability $1- s_{ij}^\kth$.

Next we give a upper bound for $\max\left(\mathbb{KL}(s^\kth_+ \Vert s^\kth_-), \mathbb{KL}(s^\kth_- \Vert s^\kth_+)\right)$. Note that
$\mathbb{KL}(s^\kth_+ \Vert s^\kth_-) 
\leq 
\frac{(s^\kth_+ - s^\kth_-)^2}{s^\kth_- (1-s^\kth_-)}
= 
\frac{(1-2r^\kth)^2 (p^\kth - q^\kth)^2}{s^\kth_- (1-s^\kth_-)}$,
and
$\mathbb{KL}(s^\kth_- \Vert s^\kth_+) 
\leq 
\frac{(s^\kth_- - s^\kth_+)^2}{s^\kth_+ (1-s^\kth_+)}
= 
\frac{(1-2r^\kth)^2 (p^\kth - q^\kth)^2}{s^\kth_+ (1-s^\kth_+)}$. This lead to 
\[
\max\left(\mathbb{KL}(s^\kth_+ \Vert s^\kth_-), \mathbb{KL}(s^\kth_- \Vert s^\kth_+)\right)
\leq 
\frac{(1-2r^\kth)^2 (p^\kth - q^\kth)^2}{\min\left(s^\kth_+ (1-s^\kth_+), s^\kth_- (1-s^\kth_-)\right)} \,.
\]

Now going back to the Fano's inequality, we obtain
\begin{align*}
\Prob{\hat{Y}^\kth \neq Y^{\kth \ast}} 
&\geq 1 - \frac{I(Y^{\kth \ast}, W^\kth) + \log 2}{n^\kth \log 2} \\
&\geq 1 - \frac{\abs{\Omega^\kth} \cdot \frac{(1-2r^\kth)^2 (p^\kth - q^\kth)^2}{\min\left(s^\kth_+ (1-s^\kth_+), s^\kth_- (1-s^\kth_-)\right)} + \log 2}{n^\kth \log 2} \\
&\geq \frac{1}{2}\,,
\end{align*}
where last inequality holds given that 
$\frac{(1-2r^\kth)^2 (p^\kth - q^\kth)^2}{\min\left(s^\kth_+ (1-s^\kth_+), s^\kth_- (1-s^\kth_-)\right)}
\leq 
\frac{n^\kth}{2 \abs{\Omega^\kth}}$. This completes our proof.
\end{proof}

%%%%%%%%%%%%%%%%%%%%%%%%%%%%%%%

\end{document}